\documentclass[letterpaper, 10pt, conference]{ieeeconf}      

\IEEEoverridecommandlockouts                              

\usepackage{cite}
\usepackage{amsmath,amssymb,amsfonts}
\usepackage{graphicx}
\usepackage{censor}

\newcommand{\censorNormal}[1]{%
    \begingroup
    \censor{#1}%
    \endgroup}

\newcommand{\censorFootnote}[1]{%
    \begingroup
    \setlength{\censorruleheight}{1.5ex}%
    \setlength{\censorruledepth}{0ex}%
    \blackout{#1}%
    \endgroup}

\StopCensoring

\usepackage{mathtools}
\usepackage{booktabs}
\usepackage{siunitx}
\usepackage{csquotes}
\MakeOuterQuote{"} 
\usepackage{placeins}

\newtheorem{theorem}{Theorem}
\newtheorem{remark}{Remark}
\newtheorem{corollary}{Corollary}

\newtheorem{definition}{Definition}
\newtheorem{example}{Example}

\newcommand{\changed}[1]{\textcolor{black}{#1}}
\newcommand{\PS}[1]{\textcolor{black}{#1}}

\newcommand{\Wedge}[1]{\left(#1\right)_{\times}}

\newcommand{\reals}{\mathbb{R}}
\newcommand{\SO}[1]{SO(#1)}

\newcommand{\SE}[3]{SE_{#2}^{#3}(#1)}
\newcommand{\se}[3]{\mathfrak{se}_{#2}^{#3}(#1)}

\newcommand{\bfzero}{\mathbf{0}}

\newcommand{\Ad}[1]{\mathbf{Ad}_{#1}}

\newcommand{\F}[1]{\mathcal{F}\mskip-5mu_{\scriptstyle #1}}
\newcommand{\rot}[2]{\bfR_{#1}^{\scriptstyle #2}}
\newcommand{\hrot}[2]{\hat{\bfR}_{#1}^{\scriptstyle #2}}
\newcommand{\inp}[2]{\bfu_{#1}^{\scriptstyle #2}}

\newcommand{\sigvec}[2]{\prescript{\vphantom{|}\scriptstyle #2}{}{\mathrlap{\bovarsigma}\hphantom{\bovarsigma}}^{\vphantom{|}\scriptstyle #1}}
\newcommand{\rhovec}[2]{\prescript{\vphantom{|}\scriptstyle #2}{}{\mathrlap{\borho}\hphantom{\bovarsigma}}^{\vphantom{|}\scriptstyle #1}}
\newcommand{\thetavec}[2]{\prescript{\vphantom{|}\scriptstyle #2}{}{\mathrlap{\botheta}\hphantom{\bovarsigma}}^{\vphantom{|}\scriptstyle #1}}

\newcommand{\w}[2]{\bfw_{#1}^{\scriptstyle #2}}
\newcommand{\n}[3]{\prescript{\vphantom{|}\scriptstyle #3}{\vphantom{#1}}{\bfn}_{#1}^{\vphantom{|}\scriptstyle #2}}
\newcommand{\pos}[3]{\prescript{\vphantom{|}\scriptstyle #3}{\vphantom{#1}}{\bfp}_{#1}^{\vphantom{|}\scriptstyle #2}}
\newcommand{\vel}[3]{\prescript{\vphantom{|}\scriptstyle #3}{\vphantom{#1}}{\bfv}_{#1}^{\vphantom{|}\scriptstyle #2}}
\newcommand{\hpos}[3]{\prescript{\vphantom{|}\scriptstyle #3}{\vphantom{#1}}{\hat{\bfp}}_{#1}^{\vphantom{|}\scriptstyle #2}}
\newcommand{\hvel}[3]{\prescript{\vphantom{|}\scriptstyle #3}{\vphantom{#1}}{\hat{\bfv}}_{#1}^{\vphantom{|}\scriptstyle #2}}
\newcommand{\acc}[3]{\prescript{\vphantom{|}\scriptstyle #3}{\vphantom{#1}}{\bfa}_{#1}^{\vphantom{|}\scriptstyle #2}}

\newcommand{\gyr}[3]{\prescript{\vphantom{|}\scriptstyle #3}{\vphantom{#1}}{\boomega}_{#1}^{\vphantom{|}\scriptstyle #2}}

\newcommand{\extpose}[3]{\prescript{\vphantom{|}\scriptstyle #3}{\vphantom{#1}}{\bochi}_{#1}^{\vphantom{|}\scriptstyle #2}}
\newcommand{\hextpose}[3]{\prescript{\vphantom{|}\scriptstyle #3}{\vphantom{#1}}{\hat{\bochi}}_{#1}^{\vphantom{|}\scriptstyle #2}}
\newcommand{\Q}[2]{\bfQ_{#1}^{\scriptstyle #2}}
\newcommand{\N}[2]{\bfN_{#1}^{\scriptstyle #2}}
\newcommand{\g}[1]{\prescript{\vphantom{|}\scriptstyle #1}{\vphantom{|}}{\bfg}}
\newcommand{\X}[3]{\prescript{\vphantom{|}\scriptstyle #3}{\vphantom{#1}}{\bfX}_{#1}^{\vphantom{|}\scriptstyle #2}}
\newcommand{\hX}[3]{\prescript{\vphantom{|}\scriptstyle #3}{\vphantom{#1}}{\hat{\bfX}}_{#1}^{\vphantom{|}\scriptstyle #2}}
\newcommand{\G}[3]{\prescript{\vphantom{|}\scriptstyle #3}{\vphantom{#1}}{\bfG}_{#1}^{\vphantom{|}\scriptstyle #2}}

\newcommand{\Ome}[3]{\prescript{\vphantom{|}\scriptstyle #3}{\vphantom{#1}}{\boOmega}_{#1}^{\vphantom{|}\scriptstyle #2}}

\DeclareMathOperator{\blkdiag}{blkdiag}

\newcommand{\bfa}{\mathbf{a}}

\newcommand{\bfb}{\mathbf{b}}

\newcommand{\bfd}{\mathbf{d}}

\newcommand{\bfe}{\mathbf{e}}
\newcommand{\bbfe}{\bar{\mathbf{e}}}

\newcommand{\bff}{\mathbf{f}}
\newcommand{\bbff}{\bar{\mathbf{f}}}

\newcommand{\bfG}{\mathbf{G}}

\newcommand{\bfg}{\mathbf{g}}

\newcommand{\bfI}{\mathbf{I}}

\newcommand{\calJ}{\mathcal{J}}

\newcommand{\calL}{\mathcal{L}}

\newcommand{\bfN}{\mathbf{N}}

\newcommand{\calN}{\mathcal{N}}

\newcommand{\bfn}{\mathbf{n}}

\newcommand{\bfp}{\mathbf{p}}

\newcommand{\bfQ}{\mathbf{Q}}

\newcommand{\bfR}{\mathbf{R}}
\newcommand{\bfX}{\mathbf{X}}

\newcommand{\bfu}{\mathbf{u}}

\newcommand{\bfv}{\mathbf{v}}

\newcommand{\bfw}{\mathbf{w}}

\newcommand{\bfx}{\mathbf{x}}

\newcommand{\bfy}{\mathbf{y}}

\newcommand{\bochi}{\protect\raisebox{1pt}{$\boldsymbol{\chi}$}}

\newcommand{\boGamma}{\boldsymbol{\Gamma}}

\newcommand{\boomega}{\boldsymbol{\omega}}

\newcommand{\boOmega}{\boldsymbol{\Omega}}
\newcommand{\bomu}{\boldsymbol{\mu}}

\newcommand{\bonu}{\boldsymbol{\nu}}
\newcommand{\bophi}{\boldsymbol{\phi}}

\newcommand{\boPhi}{\boldsymbol{\Phi}}
\newcommand{\borho}{\boldsymbol{\rho}}

\newcommand{\bovarsigma}{\boldsymbol{\varsigma}}

\newcommand{\botheta}{\boldsymbol{\theta}}

\newcommand{\boUpsilon}{\boldsymbol{\Upsilon}}

\newcommand{\boxi}{\boldsymbol{\xi}}

\newcommand{\bboxi}{\bar{\boxi}}

\newcommand{\boLambda}{\boldsymbol{\Lambda}}

\newcommand{\EKF}{\textcolor{black}{EKF}}
\newcommand{\IEKF}{\textcolor{black}{IEKF}}

\newcommand{\IIEKF}{\textcolor{black}{IterIEKF}}
\newcommand{\IterEKF}{\textcolor{black}{IterEKF}}

\usepackage{xcolor}
\definecolor{mplC0}{HTML}{1F77B4}
\definecolor{mplC1}{HTML}{FF7F0E}
\definecolor{mplC2}{HTML}{2CA02C}

\begin{document}

\title{\LARGE \bf
Invariant Kalman filtering for extended pose estimation \\in multi-IMU articulated rigid-body systems
}

\author{
    \censorNormal{Sven Goffin}$^1$, 
    \censorNormal{Cédric Schwartz}$^2$,
    \censorNormal{Silvère Bonnabel}$^3$, 
    \censorNormal{Olivier Brüls}$^4$, 
    and 
    \censorNormal{Pierre Sacré}$^1$
    \thanks{\censorFootnote{S.\ Goffin} is a \censorFootnote{FRIA grantee of the Fonds de la Recherche Scientifique - FNRS}.}
    \thanks{$^{1}$\censorFootnote{S.\ Goffin} and \censorFootnote{P.\ Sacré} are with the \censorFootnote{Department of Electrical Engineering and Computer Science, University of Liège, Belgium (sven.goffin@uliege.be; p.sacre@uliege.be)}.}
    \thanks{$^{2}$\censorFootnote{C.\ Schwartz} is with the \censorFootnote{Department of Physical Activity and Rehabilitation Sciences, University of Liège, Belgium (Cedric.Schwartz@uliege.be)}.}
    \thanks{$^{3}$\censorFootnote{S.\ Bonnabel} is with the \censorFootnote{Department of Mathematics and Systems, Mines Paris -- PSL, France (silvere.bonnabel@mines-paristech.fr)}.}
    \thanks{$^{4}$\censorFootnote{O.\ Brüls} is with the \censorFootnote{Department of Aerospace and Mechanical Engineering, University of Liège, Belgium (o.bruls@uliege.be)}.} 
}
\date{}
\maketitle


\begin{abstract}
    Accurate extended pose estimation (orientation, velocity, and position) for IMU-instrumented articulated rigid-body systems is a key challenge in  robotics and human motion analysis. 
    The invariant extended Kalman filter (\IEKF{}) addresses this problem for a single rigid body with convergence guarantees and consistency under unobservability, but extending these properties to articulated systems is nontrivial: inter-body pose coupling prevents a direct application, and incorporating joint kinematic constraints within the invariant framework remains an open problem. 
    To address this gap, we introduce the relative $L$-extended pose, a Lie group representation for kinematic-tree systems \changed{based on relative poses between bodies}. With one IMU per body, it yields group-affine dynamics and allows joint constraints to be expressed in invariant form. We incorporate these constraints as noise-free pseudo-measurements within an iterated \IEKF{} (\IIEKF{}), thereby preserving the convergence and consistency guarantees of invariant filtering. 
    \PS{The absolute pose of each body is then recovered by chaining the relative blocks along the kinematic tree.}
    Validated on both a UR5e robot and a human leg, the proposed \IIEKF{} outperforms all \EKF{}, \IterEKF{}, and absolute-pose \IIEKF{} baselines. It converges faster, exhibits lower run-to-run variability, and consistently achieves the lowest RMSE, with reductions of at least $50\%$ compared to the second-best filter across all scenarios considered in this work.
\end{abstract}

\section{Introduction}
\label{sec:intro}

Wearable inertial measurement units (IMUs) have become
the sensors of choice for motion tracking in robotics and human movement analysis
\cite{field2011human}. Unlike vision-based systems, which require bulky external
infrastructure, have a limited capture volume, and are sensitive to
occlusions, IMUs are small and self-contained, enabling motion
capture in constrained or open environments. As a result, IMU-based pose
estimation has found broad applications in exoskeletons \cite{xavier2023multi},
legged robots \cite{bloesch2013state, rotella2014state, lin2006sensor}, and
human motion analysis \cite{filippeschi2017survey, roetenberg2009xsens,
baldi2019upper}.

Stochastic filtering is the standard framework for IMU-based pose estimation,
with the extended Kalman filter (\EKF{})~\cite{Kalman1960new} and its variants
being the most widely used methods for articulated systems
\cite{pastorino2013state}. These methods, however, do not exploit the Lie group
structure of rigid-body motion. Geometric filtering addresses this limitation
\cite{bonnabel2005invariant, mahony2008nonlinear, mahony2021equivariant, van2019geometric, mahony2021homogeneous}. In particular,
the invariant extended Kalman filter (\IEKF{}) \cite{barrau2016invariant,
barrau2018invariant} achieves convergence guarantees and consistency under
unobservability for single-body extended pose (orientation, velocity, position)
estimation \cite{barrau2018invariant, van2020invariant, hartley2020contact,
pavlasek2021invariant, wu2017invariant, barrau2022geometry}. Extending these
properties to articulated rigid-body systems remains, however, an open problem:
inter-body pose coupling prevents a direct application of the \IEKF{} to such
systems, and incorporating kinematic constraints while preserving its convergence
and consistency guarantees is nontrivial.

In this paper, we reformulate the problem of extended pose estimation for kinematic-tree articulated systems as an invariant filtering problem and we advocate the use of the \emph{relative $L$-extended pose} as a Lie group state representation for such systems equipped with $L$ IMUs (Fig.~\ref{fig:rigidbody_system} shows $L=2$). This representation yields group-affine dynamics and allows joint constraints to be expressed in invariant form, enabling their incorporation as noise-free pseudo-measurements within an iterated \IEKF{} (\IIEKF{}) \cite{goffin2025iterated}. The result is a principled extension of invariant filtering to articulated rigid-body systems, with the associated convergence and consistency guarantees. The proposed method is experimentally validated on two extended-pose estimation tasks: a UR5e pick-and-place sequence and a human-leg forward-lunge sequence. In both cases, the proposed \IIEKF{} outperforms an \IIEKF{} based on absolute body poses, as well as all \EKF{} and iterated \EKF{} (\IterEKF{}) baselines under either representation. 

\begin{figure}
    \centering
    \includegraphics[width=\linewidth]{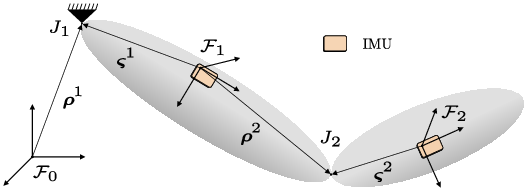}
    \caption{%
    Rigid-body system with a tree structure, composed of two movable bodies and a base fixed in $\F{0}$. An IMU is mounted on each movable body. 
    Our goal is to estimate the extended pose (position, velocity, orientation) of the full articulated system, from \PS{IMU measurements and a calibrated kinematic model, without relying on online external pose measurements}.}
    \label{fig:rigidbody_system}
\end{figure}

The contributions of this paper are: (i) the introduction of the relative $L$-extended pose, a Lie group state representation for multi-IMU articulated rigid-body systems, which admits group-affine dynamics and invariant joint constraints; (ii) an \IIEKF{} formulation based on the proposed representation that incorporates spherical and hinge joint constraints as noise-free pseudo-measurements; and (iii) an experimental evaluation on a robotic arm and a human-leg sequence, validating the approach across diverse articulated systems.

\section{Problem statement}
\label{sec:problem_statement}

We consider an articulated system of $L$ rigid bodies arranged in a kinematic tree, with one rigidly mounted IMU per body (Fig.~\ref{fig:rigidbody_system} shows $L=2$). Bodies and joints follow Featherstone's indexing \cite{featherstone2008rigid}: body $0$ is the fixed base, anchored in the inertial frame $\F{0}$, and frame $\F{i}$ denotes the frame of body $i$, which we take to coincide with its IMU frame without loss of generality. For joint $J_j$, $p(j)$ and $s(j)$ are the predecessor and successor body indices, and $\rhovec{j}{}, \sigvec{j}{} \in \reals^3$ are the vectors from $\F{p(j)}$ and $\F{s(j)}$ to $J_j$. 
For any vector or matrix quantity, the bottom-right index is the time step and the top-right superscript identifies the frames involved. If specified, the top-left superscript indicates the expression frame. This way, $\rot{k}{ij}\in\SO{3}$ is the rotation from $\F{j}$ to $\F{i}$ at time $k$, and $\vel{k}{ij}{l}, \pos{k}{ij}{l}\in\reals^3$ are the velocity and position of $\F{j}$ with respect to (w.r.t.) $\F{i}$, expressed in $\F{l}$.

Neglecting IMU biases and without motion priors, each body follows the 
inertial navigation dynamics
\begin{subequations}\label{eq:inertial_dyn}
    \begin{align}
        \rot{k+1}{0i} &= \rot{k}{0i}\exp_{\SO{3}}((\gyr{k}{i}{i} 
            + \w{k}{i,\boomega}) dt),\\
        \vel{k+1}{0i}{0} &= \vel{k}{0i}{0} + \left(\rot{k}{0i} 
            (\acc{k}{i}{i} + \w{k}{i,\bfa}) + \g{0}\right)dt,\\
        \pos{k+1}{0i}{0} &= \pos{k}{0i}{0} + \vel{k}{0i}{0} dt
            + \left(\rot{k}{0i} (\acc{k}{i}{i} + \w{k}{i,\bfa}) 
            + \g{0}\right)\frac{dt^2}{2},
    \end{align}
\end{subequations}
where $\gyr{k}{i}{i}, \acc{k}{i}{i} \in \reals^3$ are the IMU angular 
velocity and specific force in $\F{i}$, $\g{0}$ is gravity in $\F{0}$, 
$dt$ is the time step, and $\w{k}{i}=(\w{k}{i,\boomega}, 
\w{k}{i,\bfa})\sim\calN(\bfzero_{6\times 1}, \Q{k}{i})$ is the process 
noise. Joints couple the body poses through kinematic constraints; for 
the spherical joints in Fig.~\ref{fig:rigidbody_system}:
\begin{subequations}\label{eq:constraints}
    \begin{align}
        J_1:\quad \pos{k}{01}{0} + \rot{k}{01}\sigvec{1}{1} 
            &= \rhovec{1}{0}, \label{eq:constraint_J1}\\
        J_2:\quad \pos{k}{02}{0} + \rot{k}{02}\sigvec{2}{2} 
            &= \pos{k}{01}{0} + \rot{k}{01}\rhovec{2}{1}, 
            \label{eq:constraint_J2}
    \end{align}
\end{subequations}
where $\rhovec{j}{p(j)}, \sigvec{j}{s(j)}\in \reals^3$ with $j\in \{1,2\}$ are known from calibration. 
\PS{\textbf{Our goal is to estimate the extended pose of each body (orientation, velocity, position in the inertial frame $\F{0}$)  \changed{using IMU measurements and a calibrated kinematic model, without relying on online external pose measurements.}}}

Two state representations are commonly used to address this problem, each with a fundamental limitation. In the free-segment model, each  body has an independent pose w.r.t.~$\F{0}$ and joint constraints are imposed as stochastic pseudo-measurements. 
\changed{This approach cannot properly handle the noise-free constraint \eqref{eq:constraint_J2}, as it couples the poses of two distinct bodies. Enforcing this constraint as a noisy measurement in an \EKF{} update \cite{roetenberg2009xsens} leaves residual constraint violations \cite{goffin2023invariant}, while handling it via a particle filter \cite{zhang2011novel} incurs substantially higher computational cost.} The constrained-optimization approach \cite{kok2014optimization} shares the same difficulty.

The kinematic-tree model avoids this issue by combining the root pose w.r.t.\ $\F{0}$ with relative poses between adjacent segments in minimal coordinates, so that constraints are satisfied by construction. In \cite{bleser2015cognitive}, relative orientations are parameterized by Euler angles and, together with their first and second derivatives, are estimated with an \EKF{} using a constant angular-acceleration model and an IMU-based measurement model. \cite{miezal2013generic} applies the same idea using the Denavit-Hartenberg hinge-joint rotation angle. The resulting dynamics of such approach are often a poor approximation of real motion, and the measurement model is highly nonlinear, in such a way that repeated linearization about the current estimate can hinder convergence.

Those limitations ultimately stem from the choice of state representation. A Lie group representation of the state compatible with the invariant filtering framework would overcome the shortcomings of both the free-segment and kinematic-tree models. As we show in the next section, the relative $L$-extended pose provides such a representation.


\section{Modeling the extended pose of rigid-body systems within the invariant framework}
\label{sec:modelling_extpose}

The invariant filtering framework requires the state $\X{k}{}{}$ to evolve on a matrix Lie group and imposes two conditions on the system model \cite{barrau2016invariant}:
\begin{enumerate}
    \item \textit{Group-affine dynamics:} 
    \[\X{k+1}{}{} = \bff(\X{k}{}{}, \inp{k}{}, \w{k}{}),\] 
    where, up to first order in the process noise $\bfw_k \sim \calN(\bfzero, \Q{k}{})$, $\bff$ admits the factorization $\bff(\X{k}{}{},\bfu_k,\bfw_k) = \bbff(\X{k}{}{},\bfu_k)\bfg(\bfw_k,\bfu_k)$, with $\bfg(\bfzero,\bfu_k)=\bfI$ and
    \begin{equation}\label{eq:group_affine}
        \bbff(\bomu\bonu,\bfu_k) = 
            \bbff(\bomu,\bfu_k)\bbff(\bfI,\bfu_k)^{-1}\bbff(\bonu,\bfu_k),
    \end{equation}
    for all inputs $\inp{k}{}$ and all states $\bomu,\bonu$. Here, $\bfzero$ and $\bfI$ denote the zero vector and identity matrix of appropriate dimension.
    
    \item \textit{Measurements in invariant form:}
    \begin{equation}\label{eq:invariant_form}
        \bfy_k = \X{k}{}{}\bfd_k + \n{k}{}{} \quad \text{or} \quad 
        \bfy_k = \X{k}{-1}{}\bfd_k + \n{k}{}{},
    \end{equation}
    with $\n{k}{}{} \sim \calN(\bfzero, \N{k}{})$ and $\bfd_k$ a known column vector.
\end{enumerate}
The key contribution of this section is to show that the relative $L$-extended pose, defined below, satisfies both conditions for the articulated system described in Section~\ref{sec:problem_statement}, thereby casting the problem into the invariant filtering framework.

\subsection{\PS{The relative L-extended pose}}
\label{subsec:relpose}

In single-body pose estimation, the \IEKF{} represents the extended pose of body $i$ w.r.t.\ $\F{0}$, expressed in $\F{0}$, by the matrix $\extpose{k}{0i}{0} \in \SE{3}{2}{}$, where $\SE{3}{2}{}$ is the matrix Lie group of extended poses \cite{barrau2015non}:
\begin{equation}
    \extpose{k}{0i}{0} \coloneq
    \left[
    \begin{array}{c|c}
        \rot{k}{0i} & 
        \begin{matrix}
            \vel{k}{0i}{0} & \pos{k}{0i}{0}
        \end{matrix} \\
        \hline
        \bfzero_{2\times 3} 
        & \bfI_{2}
    \end{array}
    \right] \in \SE{3}{2}{},
\end{equation}
where $\bfI_N \in \reals^{N\times N}$ and $\bfzero_{N\times M} \in 
\reals^{N\times M}$ denote the identity and zero matrices. Using 
$\SE{3}{2}{}$ as a building block, we define the group of $L$-extended 
poses as follows.

\begin{definition}
    The Lie group of $L$-extended poses is the set of block-diagonal matrices in $\reals^{5L \times 5L}$ composed of $L$ individual elements of $\SE{3}{2}{}$:
    \begin{equation*}
        \SE{3}{2}{L} \coloneq \left\{  \blkdiag(\bochi_1, \dots, \bochi_L)\left| \begin{array}{c}
            \bochi_i \in \SE{3}{2}{}, \\
            i=1,\dots,L
        \end{array}\right.\right\}.
    \end{equation*}
    The exponential map and Lie algebra identification on $\SE{3}{2}{L}$ 
    are induced blockwise by those of $\SE{3}{2}{}$. Specifically, define the block-diagonal lifting operator 
    $(\cdot)^\boxdot$ for any function $\bff_{\mathrm{vec}} : \reals^n \rightarrow \reals^{N 
    \times N}$ or $\bff_{\mathrm{mat}} : \reals^{N\times N} \rightarrow 
    \reals^{N\times N}$ by
    \begin{align*}
        \bff_{\mathrm{vec}}^{\boxdot}(\bfx) &\coloneq 
            \blkdiag\left(\bff_{\mathrm{vec}}(\bfx_1),\dots,
            \bff_{\mathrm{vec}}(\bfx_L)\right),\\
        \bff_{\mathrm{mat}}^{\boxdot}(\bfX) &\coloneq 
            \blkdiag\left(\bff_{\mathrm{mat}}(\bfX_1),\dots,
            \bff_{\mathrm{mat}}(\bfX_L)\right),
    \end{align*}
    for any vector $\bfx = (\bfx_1,\dots,\bfx_L)\in\reals^{Ln}$ and matrix $\bfX = \blkdiag(\bfX_1,\dots,\bfX_L) \in  \reals^{LN\times LN}$. Then:
    \begin{equation*}
        \exp_{\SE{3}{2}{L}} \coloneq \exp_{\SE{3}{2}{}}^\boxdot, \quad
        \calL_{\se{3}{2}{L}} \coloneq \calL_{\se{3}{2}{}}^\boxdot.
    \end{equation*}
    See Appendix~\ref{appendix} for the definition of $\calL_{\se{3}{2}{}}$.
\end{definition} 


\changed{Although a block-diagonal state built from \emph{absolute} poses (i.e. the free-segment model) inherits group-affine dynamics blockwise \cite{brossard2021associating}, it prevents constraints coupling two bodies w.r.t.\ $\F{0}$ from being written in invariant form. This motivates working instead with relative poses.}

\begin{definition}[Relative $L$-extended pose]
    The relative $L$-extended pose of a rigid-body system arranged as a kinematic tree is the element $\X{k}{}{} \in \SE{3}{2}{L}$ given by
    \begin{equation*}
        \X{k}{}{} = \blkdiag\!\left(\extpose{k}{p(1)s(1)}{p(1)}, \dots, 
        \extpose{k}{p(L)s(L)}{p(L)}\right),
    \end{equation*}
    where the $j^{\mathrm{th}}$ diagonal block $(\X{k}{}{})_j = \extpose{k}{p(j)s(j)}{p(j)}$ is the extended pose of $\F{s(j)}$ w.r.t.\ $\F{p(j)}$, expressed in $\F{p(j)}$.
\end{definition}

This definition combines the structural simplicity of relative coordinates with the geometric consistency of $\SE{3}{2}{L}$. 


{\color{black}
\begin{example}[Two-link system]\label{ex:twolink_def}
        For the system of Fig.~\ref{fig:rigidbody_system}, with $L=2$, Featherstone's indexing gives $p(1)=0$, $s(1)=1$ and $p(2)=1$, $s(2)=2$. Its relative $2$-extended pose is
    \begin{equation*}
        \X{k}{}{} = \blkdiag\!\left(\extpose{k}{01}{0}, \extpose{k}{12}{1}\right),
    \end{equation*}
    where $\extpose{k}{01}{0}$ is the extended pose of $\F{1}$ w.r.t.\ the fixed base $\F{0}$, and $\extpose{k}{12}{1}$ is the extended pose of $\F{2}$ w.r.t.\ $\F{1}$, expressed in $\F{1}$.
    The absolute poses are recovered by chaining the relative blocks along the corresponding paths in the kinematic tree:
    \begin{equation}\label{eq:abs_recovery}
        \extpose{k}{01}{0} = (\X{k}{}{})_1, \qquad
        \extpose{k}{02}{0} = (\X{k}{}{})_1 (\X{k}{}{})_2.
    \end{equation}
\end{example}

}

\subsection{\changed{The relative L-extended pose has group-affine dynamics }}
\label{subsec:dynamics}

{
\color{black}
We first derive the two-link relative dynamics explicitly and introduce the compact notation used in the general case. Theorem~\ref{thm:dyn} then generalizes this derivation to an arbitrary kinematic tree. Throughout, relative and absolute poses are related by
\begin{equation}\label{eq:rel_extpose}
    \extpose{k}{p(j)s(j)}{p(j)} =
    \left(\extpose{k}{0p(j)}{0}\right)^{-1}
    \left(\extpose{k}{0s(j)}{0}\right).
\end{equation}

\begin{example}[Two-link dynamics]\label{ex:twolink_dyn}
    Continuing Example~\ref{ex:twolink_def}, the dynamics of $\extpose{k}{01}{0}$ are given by \eqref{eq:inertial_dyn} and are known to be group-affine \cite{barrau2015non}. Substituting \eqref{eq:inertial_dyn} into \eqref{eq:rel_extpose} yields
    \begin{subequations}\label{eq:rel_dyn}
        \begin{align}
            \rot{k+1}{12}    &= \Ome{k}{1}{1} \rot{k}{12} \left(\Ome{k}{2}{2}\right)^T,\\
            \vel{k+1}{12}{1} &= \Ome{k}{1}{1} \left( \vel{k}{12}{1} + \acc{k}{rel}{1} dt\right),\\
            \pos{k+1}{12}{1} &= \Ome{k}{1}{1} \left( \pos{k}{12}{1} + \vel{k}{12}{1} dt + \acc{k}{rel}{1} \frac{dt^2}{2} \right),
        \end{align}
    \end{subequations}
    where $\Ome{k}{i}{i}= \exp_{\SO{3}}\left(-(\gyr{k}{i}{i} + \w{k}{i,\boomega}) dt\right)$ and $\acc{k}{rel}{1} = \rot{k}{12}(\acc{k}{2}{2} + \w{k}{2,\bfa}) - (\acc{k}{1}{1} + \w{k}{1,\bfa})$.
    The expanded form \eqref{eq:rel_dyn} is explicit but cumbersome. Following \cite{brossard2021associating}, \eqref{eq:inertial_dyn} admits the following first-order approximation in the IMU noise:
    \begin{equation}\label{eq:ext_pose_dyn}
        \extpose{k+1}{01}{0}
        = \boGamma\boPhi(\extpose{k}{01}{0})\boUpsilon(\inp{k}{1})
        \exp_{\SE{3}{2}{}}(\G{k}{1}{} \w{k}{1}),
    \end{equation}
    where $\inp{k}{1}=(\gyr{k}{1}{1}, \acc{k}{1}{1})$ and $\boPhi$ is an automorphism. The matrices $\boGamma$, $\boPhi(\extpose{k}{01}{0})$, $\boUpsilon(\inp{k}{1})\in\reals^{5\times 5}$, and $\G{k}{1}{}\in\reals^{9\times 6}$ are given in Appendix~\ref{appendix}. Substituting \eqref{eq:ext_pose_dyn} into \eqref{eq:rel_extpose} gives, up to first order in the IMU noise,
    \begin{align}\label{eq:rel_ext_pose_dyn}
        \extpose{k+1}{12}{1}
        &= \Delta \exp_{\SE{3}{2}{}}\left(\G{k}{2}{} \w{k}{2} - \Ad{\Delta^{-1}}\G{k}{1}{} \w{k}{1}\right),
    \end{align}
    with $\Delta = \boUpsilon(\inp{k}{1})^{-1}\boPhi(\extpose{k}{12}{1})\boUpsilon(\inp{k}{2})$. Here, $\Ad{\Delta^{-1}}$ denotes the adjoint of $\Delta^{-1}$ defined in Appendix~\ref{appendix}, and the automorphism property $\boPhi(\bomu)^{-1}\boPhi(\bonu) = \boPhi(\bomu^{-1}\bonu)$ was used.
\end{example}
 
The compact forms \eqref{eq:ext_pose_dyn} and \eqref{eq:rel_ext_pose_dyn} can be unified by treating $\extpose{k}{01}{0}$ as the relative pose between $\F{0}$ and $\F{1}$ and assigning a noise-free IMU to $\F{0}$. Then, $\w{k}{0} = \bfzero_{6\times 1}$, $\gyr{k}{0}{0} = \bfzero_{3\times 1}$, and $\acc{k}{0}{0} = -\g{0}$, so $\boUpsilon(\inp{k}{0})^{-1}$ reduces to $\boGamma$ and \eqref{eq:rel_ext_pose_dyn} takes the same form as \eqref{eq:ext_pose_dyn}. The following theorem can therefore describe all bodies, including the root, with a single formula.
}

\begin{theorem}\label{thm:dyn}
    The relative $L$-extended pose of a kinematic-tree rigid-body system admits the following first-order dynamics approximation w.r.t. the process noise $\bfw_k$:
    \begin{equation*}
        \X{k+1}{}{} = \boLambda(\X{k}{}{}, \inp{k}{p}, \inp{k}{s})\exp_{\SE{3}{2}{L}}(\bfG_k \bfw_k),
    \end{equation*}
    where
    \begin{subequations}
        \begin{align*}
            \boLambda(\X{k}{}{}, \inp{k}{p}, \inp{k}{s}) &= \boUpsilon^\boxdot(\inp{k}{p})^{-1}\boPhi^\boxdot(\X{k}{}{})\boUpsilon^\boxdot(\inp{k}{s}),\\
            \inp{k}{p} &= \left( \inp{k}{p(1)}, \dots, \inp{k}{p(L)} \right),\\
            \inp{k}{s} &= \left( \inp{k}{s(1)}, \dots, \inp{k}{s(L)} \right),\\
            \w{k}{} &= \left( \w{k}{1}, \dots, \w{k}{L} \right),
        \end{align*}
    \end{subequations}
    and the $j^{\mathrm{th}}$ sub-vector of size $9$ in $\G{k}{}{} \w{k}{}$ is given by 
    \begin{align*}
        (\G{k}{}{} \w{k}{})_j = \G{k}{s(j)}{}\w{k}{s(j)}
        - \Ad{\Delta_j^{-1}}\G{k}{p(j)}{}\w{k}{p(j)},
    \end{align*}
    with $\Delta_j = \boUpsilon(\inp{k}{p(j)})^{-1}\boPhi(\extpose{k}{p(j)s(j)}{p(j)})\boUpsilon(\inp{k}{s(j)})$.
    Moreover, the function $\boLambda(\cdot, \inp{k}{p}, \inp{k}{s})$ satisfies group-affine property \eqref{eq:group_affine}.
\end{theorem}
\PS{The proof is provided in Appendix~\ref{app:proof_th1}.}



\subsection{\changed{The relative L-extended pose yields joint constraints in invariant form}}
\label{subsec:constraints}

\begin{table*}[tp]
    \centering
    \medskip
    \caption{Spherical and hinge joint constraints expressed using the relative $L$-extended pose. The vectors $\rhovec{j}{p(j)}$, $\sigvec{j}{s(j)}$, $\thetavec{j}{p(j)}$, $\thetavec{j}{s(j)} \in \reals^3$ are known from calibration. The element $(\X{k}{}{})_j \in \SE{3}{2}{}$ denotes the $j^{\mathrm{th}}$ diagonal block of $\X{k}{}{} \in \SE{3}{2}{L}$.}
    \begin{tabular}{cl}
        \toprule
        \textbf{Joint type} &  \multicolumn{1}{c}{\textbf{Constraints in invariant form}}\\

        \midrule

        \begin{tabular}{c}
             Spherical\\
             \includegraphics[width=\columnwidth]{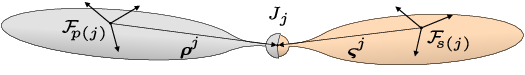}
        \end{tabular} 
        &
        \begin{tabular}{l}
             $\left(\X{k}{}{}\right)_j \begin{bmatrix} \sigvec{j}{s(j)}\\ 0\\ 1 \end{bmatrix} = \begin{bmatrix} \rhovec{j}{p(j)}\\ 0\\ 1 \end{bmatrix}$
        \end{tabular}\\[0.cm]
        
        \hline\\
        
        \begin{tabular}{c}
             Hinge\\
             \includegraphics[width=\columnwidth]{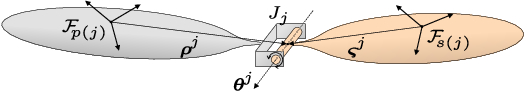}
        \end{tabular} 
        &
        \begin{tabular}{l}
            $\left(\X{k}{}{}\right)_j \begin{bmatrix} \sigvec{j}{s(j)}\\ 0\\ 1 \end{bmatrix} = \begin{bmatrix} \rhovec{j}{p(j)}\\ 0\\ 1 \end{bmatrix}$\\[0.4cm]
            $\left(\X{k}{}{}\right)_j \begin{bmatrix} \thetavec{j}{s(j)}\\ 0\\ 0 \end{bmatrix} = \begin{bmatrix} \thetavec{j}{p(j)}\\ 0\\ 0 \end{bmatrix}$
        \end{tabular}\\[1cm]    
        \bottomrule
    \end{tabular}
    \label{tab:joint_constraints}
\end{table*}

{
\color{black}

The relative $L$-extended pose removes the coupling that prevents \eqref{eq:constraint_J2} from being expressed in invariant form under the free-segment model, as illustrated by the two-link case.
\begin{example}[Two-link constraint]\label{ex:twolink_constraint}
    Left-multiplying Equation~\eqref{eq:constraint_J2} by $(\rot{k}{01})^T$ and substituting \eqref{eq:rel_extpose} with $j=2$ yields
    \begin{equation*}
        \rot{k}{12}\sigvec{2}{2} + \pos{k}{12}{1} = \rhovec{2}{1},
    \end{equation*}
    which depends only on the block $(\X{k}{}{})_2 = \extpose{k}{12}{1}$. It can be written as $(\X{k}{}{})_2 \bfd = \bfy$, where $\bfd = (\sigvec{2}{2}, 0, 1)$ and $\bfy = (\rhovec{2}{1}, 0, 1)$, yielding a measurement in invariant form. This result is formalized in the following theorem.
\end{example}
 
}

\begin{theorem}\label{thm:abs_constraints}
    Let the constraint enforced by joint $J_j$ between bodies $p(j)$ and $s(j)$ be given by
    \begin{equation}\label{eq:general_constraint}
        \bophi_k\left(\extpose{k}{0p(j)}{0}, \extpose{k}{0s(j)}{0}\right) = \bfzero.
    \end{equation}
    If $\bophi_k$ takes the form
    \begin{align}\label{eq:cons_inv_form_abs}
        \bophi_k\left(\extpose{k}{0p(j)}{0}, \extpose{k}{0s(j)}{0}\right) 
        & = \rot{k}{0p(j)}\bfa_k - \rot{k}{0s(j)}\bfb_k \notag \\
        & \quad + \alpha_k\left(\vel{k}{0p(j)}{0} - \vel{k}{0s(j)}{0}\right) \notag \\
        & \quad + \beta_k\left(\pos{k}{0p(j)}{0} - \pos{k}{0s(j)}{0}\right),
    \end{align}
    where $\bfa_k,\bfb_k\in\reals^3$ and $\alpha_k,\beta_k\in\reals$ are known and independent of the body poses ($\extpose{k}{0p(j)}{0}$, $\extpose{k}{0s(j)}{0}$, and $\extpose{k}{p(j)s(j)}{p(j)}$), then the constraint admits the invariant form based on $(\X{k}{}{})_j \in \SE{3}{2}{}$:
    \begin{align}\label{eq:inv_constraint}
        \begin{split}
            (\X{k}{}{})_j \bfd_k = \bfy_k,
        \end{split}
    \end{align}
    where $\bfd_k = (\bfb_k, \alpha_k, \beta_k)$ and $\bfy_k = (\bfa_k, \alpha_k, \beta_k)$.
\end{theorem}
\PS{The proof is provided in Appendix~\ref{app:proof_th2}.}

It may be preferable to express a constraint directly in terms of the relative pose of the connected bodies. The next corollary states the counterpart of Theorem~\ref{thm:abs_constraints} in this setting.

\begin{corollary}\label{cor:rel_constraints}
    Let the constraint enforced by joint $J_j$ between bodies $p(j)$ and $s(j)$ be given by
    \begin{equation}
        \bophi_k\left(\extpose{k}{p(j)s(j)}{p(j)}\right) = \bfzero.
    \end{equation}
    It admits the form \eqref{eq:inv_constraint} if $\bophi_k$ can be expressed as
    \begin{align}\label{eq:cons_inv_form_rel}
        &\bophi_k\left(\extpose{k}{p(j)s(j)}{p(j)}\right) = \rot{k}{p(j)s(j)}\bfb_k \notag\\
        &\qquad  + \alpha_k \vel{k}{p(j)s(j)}{p(j)} + \beta_k \pos{k}{p(j)s(j)}{p(j)} - \bfa_k,
    \end{align}
    where $\bfa_k,\bfb_k\in\reals^3$ and $\alpha_k,\beta_k\in\reals$ are known and independent of
    $\extpose{k}{p(j)s(j)}{p(j)}$.
\end{corollary}

Both spherical-joint constraints in \eqref{eq:constraints} match \eqref{eq:cons_inv_form_abs} and thus admit an invariant form under the relative $L$-extended pose. 
Table~\ref{tab:joint_constraints} reports the corresponding invariant form for spherical joints, as well as for hinge joints.

\begin{remark}
    In order to impose non-holonomic constraints, $\bfa_k$, $\bfb_k$, $\alpha_k$, and $\beta_k$ may depend on the gyroscope measurements $\gyr{k}{p(j)}{p(j)}$ and $\gyr{k}{s(j)}{s(j)}$, in which case the pseudo-measurement \eqref{eq:inv_constraint} should include an additional noise term to account for gyroscope noise.
\end{remark}

\section{Experiments}\label{sec:experiments}

We evaluate an \IIEKF{}, an \EKF{}, and an \IterEKF{} on two real-world extended-pose estimation tasks: a UR5e robot performing pick-and-place motions and a human leg during a forward-lunge exercise. For each filter, the state is represented either as the relative $L$-extended pose or as the collection of individual body poses w.r.t. $\F{0}$ (free-segment model), referred to as the relative and absolute representations, respectively. Quantities expressed in the absolute representation are denoted with a bar.


The \IIEKF{} is formulated using the left-invariant error:
\begin{align}
    (\boxi_{k \mid l})_j &= \log_{\SE{3}{2}{}}\left( (\hX{k\mid l}{}{})_j^{-1} (\X{k}{}{})_j\right), \label{eq:IEKF_err_rel}\\
    (\bboxi_{k \mid l})_j &= \log_{\SE{3}{2}{}}\left( (\hextpose{k\mid l}{0s(j)}{0})^{-1} \;\extpose{k}{0s(j)}{0}\right), \label{eq:IEKF_err_abs}
\end{align}
while the \EKF{} and \IterEKF{} use the standard error
\begin{align}
    (\bfe_{k \mid l})_j 
    &= \begin{bmatrix}
            \log_{\SO{3}}\left( (\hrot{k \mid l}{p(j)s(j)})^T \rot{k}{p(j)s(j)}\right)\\
            \vel{k}{p(j)s(j)}{p(j)} - \hvel{k \mid l}{p(j)s(j)}{p(j)}\\
            \pos{k}{p(j)s(j)}{p(j)} - \hpos{k \mid l}{p(j)s(j)}{p(j)}
        \end{bmatrix}, \label{eq:EKF_err_rel}\\
    (\bbfe_{k \mid l})_j 
    &= \begin{bmatrix}
            \log_{\SO{3}}\left( (\hrot{k \mid l}{0s(j)})^T \rot{k}{0s(j)}\right)\\
            \vel{k}{0s(j)}{0} - \hvel{k \mid l}{0s(j)}{0}\\
            \pos{k}{0s(j)}{0} - \hpos{k \mid l}{0s(j)}{0}
        \end{bmatrix}. \label{eq:EKF_err_abs}
\end{align}
All error terms are modeled as zero-mean Gaussian. To ensure a fair comparison, all filters are initialized from the same state (expressed in either absolute or relative form), with covariances chosen so that the induced state distributions match to first order. The covariance matrices are further selected such that states drawn from the initial distribution satisfy the kinematic constraints up to first order. In all experiments, joint constraints are split into positional and rotational components and imposed as pseudo-measurements with small regularization noise to avoid numerical issues (see Section~IV in \cite{goffin2025iterated}). The corresponding standard deviations are given in Tables~\ref{tab:ur5e_params} and~\ref{tab:lamh_params}.

\subsection{UR5e experiment}
\label{subsec:ur5e}

The first task estimates the extended pose of a UR5e robot. The UR5e comprises a fixed base and six movable rigid bodies in a kinematic chain, connected by hinge joints. Four Awinda (Xsens) IMUs are mounted on the first four movable bodies. The last two bodies remain fixed relative to their parents (Fig.~\ref{fig:ur5e}).
\begin{figure}[bp]
    \centering
    \includegraphics[width=\columnwidth]{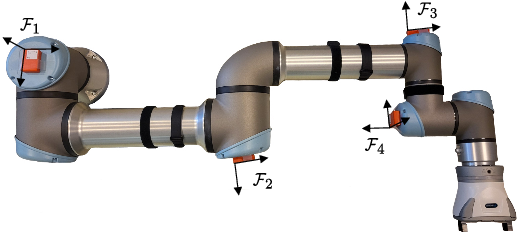}
    \caption{UR5e robot equipped with IMUs on its first four movable segments.}
    \label{fig:ur5e}
\end{figure}

Ground truth is obtained from joint-encoder measurements recorded during an arbitrary pick-and-place sequence. The vectors $\rhovec{j}{p(j)}$, $\sigvec{j}{s(j)}$, $\thetavec{j}{p(j)}$, and $\thetavec{j}{s(j)}$ (see Table~\ref{tab:joint_constraints}) are known from calibration. IMU biases are estimated offline from static data and subtracted from the measurements. 

All filters are initialized with the parameters in Table~\ref{tab:ur5e_params} and evaluated over $100$ trials on the same ground-truth trajectory, with initial states randomly sampled from the relative-representation \IIEKF{} distribution.

\begin{table}[ht]
    \centering
    \caption{Parameters for the UR5e experiment.}
    \begin{tabular}{l l}
        \toprule
        Parameter & Value \\
        \midrule
        Initial orientation error around joint axis (std.) &   $\pi/6 \, \si{\radian}$ \\
        Pos. constraint noise along each axis (std.) & $10^{-2.5} \, \si{\meter}$ \\
        Rot. constraint noise along each axis (std.) & $10^{-2.5} \, \si{\radian}$ \\
        Gyroscope noise along each axis (std.) & $10^{-2} \, \si{\radian\per\second}$ \\
        Accelerometer noise along each axis (std.) & $10^{-1} \, \si{\meter\per\second\squared}$ \\
        IMU frequency & $100 \, \si{\hertz}$\\
        \bottomrule
    \end{tabular}
    \label{tab:ur5e_params}
\end{table}

\begin{figure*}[tp]
    \centering
    \smallskip
    \includegraphics[width=\linewidth]{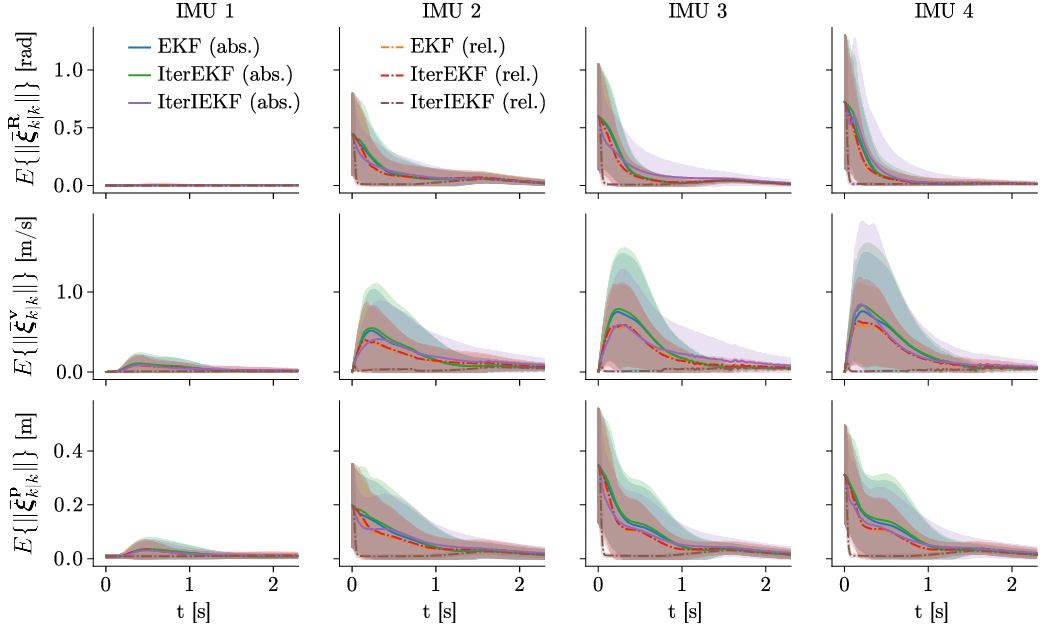}
    \caption{Mean and standard deviation of the orientation, velocity, and position error norms for each IMU over time, computed over $100$ runs of the UR5e experiment. The relative-representation \IIEKF{} (rel.) outperforms the absolute free-segment \IIEKF{} (abs.) and all \EKF{} and \IterEKF{} baselines under either extended pose representation, with negligible run-to-run variance.}
    \label{fig:ur5e_err_imu}
\end{figure*}

\changed{Since the base joint of the UR5e is aligned with gravity, its angle cannot be inferred from IMU measurements and joint constraints alone. Consequently, the rotation of IMU~1 about the gravity axis is unobservable. For all baselines, this component is therefore set offline to its ground-truth value at each time step, preventing errors in the unobservable state subspace from affecting the following comparisons.}

\changed{Figure~\ref{fig:ur5e_err_imu} reports the mean and one standard deviation of the rotation, velocity, and position error norms for the four IMUs, evaluated in the absolute representation. The orientation error of IMU~1 is identically zero due to the correction described above. The initial velocity error is zero for all filters, since the trajectory starts from rest.}

The relative-representation \IIEKF{} converges essentially immediately and with negligible run-to-run variability. All other filters converge markedly more slowly, exhibit a sharp transient increase in velocity error before gradually recovering, and show substantially larger dispersion. This indicates that the relative representation, not merely the invariant filter structure, is the decisive factor in performance.

Table~\ref{tab:ur5e_metrics} reports the root mean squared error (RMSE) for IMU orientation, velocity, and position, computed across trials, time, and IMUs. The relative-representation \IIEKF{} achieves the lowest RMSE in all three categories. In particular, it reduces RMSE by $50\%$ for orientation, $87.7\%$ for velocity, and $59.7\%$ for position compared to the relative-representation \EKF{}, which performs second best.

\begin{table}[ht]
    \centering
    \caption{RMSE of the rotation, velocity, and position components of the state for each tested filter in the UR5e experiment. Lowest values are indicated in bold. Errors are computed after correcting for the unobservable state components.}
    \begin{tabular}{lccc}
        \toprule
         & Rot. RMSE & Vel. RMSE & Pos. RMSE \\
         & [$\si{\radian}$] & [$\si{\meter \per \second}$] & [$\si{\meter}$] \\
        \midrule
        EKF (abs.)  & $0.108$ & $0.232$ & $0.068$\\
        \EKF{} (rel.)     & $0.096$ & $0.171$ & $0.057$\\
        \IterEKF{} (abs.) & $0.115$ & $0.246$ & $0.072$\\
        \IterEKF{} (rel.) & $0.100$ & $0.176$ & $0.059$\\
        IterIEKF (abs.)   & $0.119$ & $0.249$ & $0.063$\\
        \bf IterIEKF (rel.)   & $\mathbf{0.048}$ & $\mathbf{0.021}$ & $\mathbf{0.023}$\\
        \bottomrule
    \end{tabular}
    \label{tab:ur5e_metrics}
\end{table}

\subsection{Human-leg experiment}
\label{subsec:lamh}

The second task estimates the extended pose of a human leg during forward lunges. Two Trigno IMUs (Delsys) are mounted on the thigh and shank. Ground-truth IMU poses are obtained from a nine-camera Qualisys motion-capture system using reflective markers placed on rigid frames attached to each IMU. Additional markers on the malleoli, femoral epicondyles, and iliac spines are used to reconstruct joint-center trajectories, which serve as ground truth (Fig.~\ref{fig:lamh_setup}).

\begin{figure}
    \centering
    \includegraphics[width=\columnwidth]{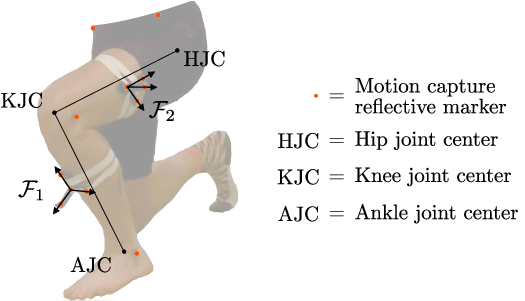}
    \caption{Human leg instrumented with IMUs on the shank and thigh. Each IMU carries a rigid marker frame to provide ground-truth IMU pose from motion capture. Eight additional markers on the malleoli, femoral epicondyles, and iliac spines provide ground-truth joint-center trajectories.}
    \label{fig:lamh_setup}
\end{figure}

The ankle and knee are modeled as hinge joints. Their centers are defined as the midpoints between the medial and lateral malleoli markers and between the medial and lateral epicondyle markers, respectively. Joint axes are estimated as the mean direction of the line joining the corresponding markers over a calibration sequence. The hip joint center is estimated using the Harrington regression from leg length and pelvic width and depth. The ankle joint center is assumed fixed in $\F{0}$, making the problem fully observable. We label the ankle and knee joints $J_1$ and $J_2$. The vectors $\rhovec{j}{p(j)}$, $\sigvec{j}{s(j)}$, $\thetavec{j}{p(j)}$, and $\thetavec{j}{s(j)}$ with $j\in\{1,2\}$ are computed from motion-capture ground truth and treated as constant. Soft-tissue artifacts induce small IMU-to-joint displacements that have limited impact on position constraints but affect rotation constraints more severely. Indeed, errors in the joint-axis direction are amplified by the joint-to-IMU lever arm, potentially leading to large position estimation errors. For this reason, the rotation-constraint regularization noise is set higher than in the UR5e experiment. IMU biases are estimated offline from ground truth and subtracted from the IMU outputs. All parameters are reported in Table~\ref{tab:lamh_params}.  We run $100$ estimations of the same trajectory with initial states sampled from the relative-representation \IIEKF{} distribution.

\begin{table}[ht]
    \centering
    \caption{Parameters for the human-leg experiment.}
    \begin{tabular}{ll}
        \toprule
        Parameter & Value \\
        \midrule
        Initial orientation error around joint axis (std.) & $\pi/3 \, \si{\radian}$ \\
        Pos. constraint noise along each axis (std.) & $10^{-2.5} \, \si{\meter}$ \\
        Rot. constraint noise along each axis (std.) & $10^{-0.5} \, \si{\radian}$ \\
        Gyroscope noise along each axis (std.) & $8\cdot10^{-3} \, \si{\radian\per\second}$ \\
        Accelerometer noise along each axis (std.) & $5\cdot10^{-2} \, \si{\meter\per\second\squared}$ \\
        IMU frequency & $200 \, \si{\hertz}$\\
        \bottomrule
    \end{tabular}
    \label{tab:lamh_params}
\end{table}

Figure~\ref{fig:lamh_err_imu} reports the mean and one standard deviation of the rotation, velocity, and position error norms for the two IMUs. The relative-representation \IIEKF{} converges quickly with negligible run-to-run variability across all state components. The other five alternatives converge more slowly, often exhibit an initial error peak, and show substantially higher variability. Overall, for each filter, the relative representation outperforms the absolute one.

\begin{figure}
    \centering
    \medskip
    \includegraphics[width=\columnwidth]{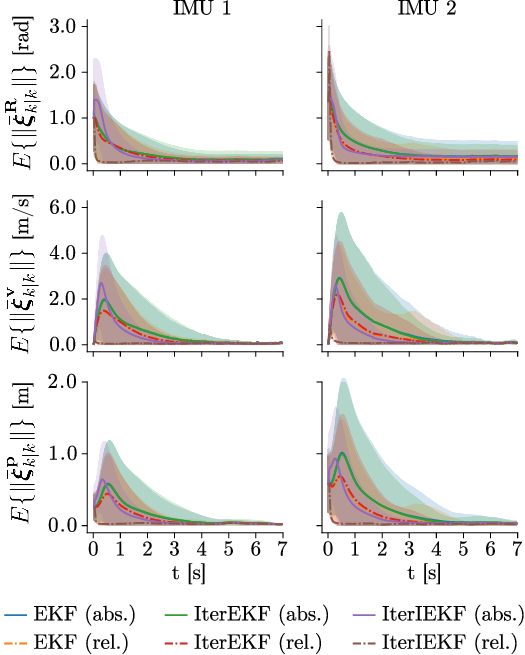}
    \caption{Mean and standard deviation of the orientation, velocity, and position error norms for each IMU, computed over $100$ runs of the human-leg experiment. The relative-representation \IIEKF{} (rel.) converges within a few time steps with negligible run-to-run variance, outperforming the absolute-representation \IIEKF{} (abs.) and all \EKF{} and \IterEKF{} baselines under either representation.}
    \label{fig:lamh_err_imu}
\end{figure}

Figure~\ref{fig:lamh_err_joints} shows the mean and standard deviation of the ankle (AJC), knee (KJC), and hip (HJC) joint-center position error norms. The same trend holds: the relative-representation \IIEKF{} performs best, converging quickly with minimal run-to-run variance. As expected, errors accumulate along the kinematic chain, increasing with distance from the ankle reference. 

After convergence, all filters retain small residual errors due to unmodeled soft-tissue artifacts and the hinge-joint assumption at the ankle and knee. For applications such as exoskeleton operation or robot programming-by-demonstration, these errors are negligible. However, more precise motion-estimation tasks will require explicitly accounting for these effects, which we leave for future work.

\begin{figure}
    \centering
    \smallskip
    \includegraphics[width=\columnwidth]{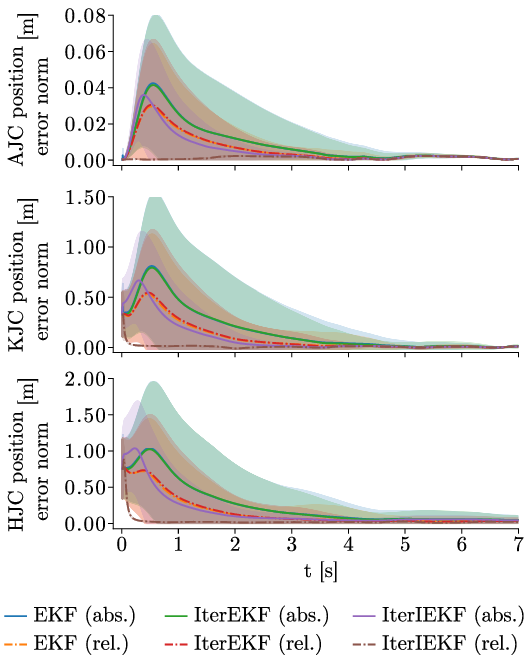}
    \caption{Mean and standard deviation of the ankle (AJC), knee (KJC), and hip (HJC) joint-center position error norms, computed over $100$ runs of the human-leg experiment. The relative-representation \IIEKF{} (rel.) converges within a few time steps and yields lower errors for all three joint centers than the absolute-representation \IIEKF{} (abs.) and all \EKF{} and \IterEKF{} baselines under either representation, with negligible run-to-run variance.}
    \label{fig:lamh_err_joints}
\end{figure}

Table~\ref{tab:lamh_metrics} presents the RMSE for IMU orientation, velocity, and position, averaged across trials, time, and IMUs. The results confirm the superiority of the relative-representation \IIEKF{}, achieving RMSE reductions of $60.7\%$, $92.4\%$, and $85.6\%$ for orientation, velocity, and position, respectively, compared to the second-best filter in each category.

\begin{table}[ht]
    \centering
    \caption{RMSE of the rotation, velocity, and position components of the state for each tested filter in the human-leg experiment. Lowest values are indicated in bold.}
    \begin{tabular}{lccc}
        \toprule
         & Rot. RMSE & Vel. RMSE & Pos. RMSE \\
         & [$\si{\radian}$] & [$\si{\meter \per \second}$] & [$\si{\meter}$] \\
        \midrule
        EKF (abs.)     & $0.289$ & $0.882$ & $0.290$\\
        \EKF{} (rel.)     & $0.230$ & $0.627$ & $0.194$\\
        \IterEKF{} (abs.) & $0.285$ & $0.860$ & $0.279$ \\
        \IterEKF{} (rel.) & $0.224$ & $0.629$ & $0.199$\\
        IterIEKF (abs.)   & $0.241$ & $0.545$ & $0.181$\\
        \bf IterIEKF (rel.)   & $\mathbf{0.088}$ & $\mathbf{0.041}$ & $\mathbf{0.026}$\\
        \bottomrule
    \end{tabular}
    \label{tab:lamh_metrics}
\end{table}


\section{Conclusion}\label{ref:conclusion}

This paper demonstrates that an appropriate Lie group representation is key to extending invariant filtering to articulated rigid-body systems. By working in relative coordinates and representing the extended pose of a kinematic-tree system as an element of $\SE{3}{2}{L}$, the filtering problem naturally acquires the group-affine dynamics and invariant measurement structure required by the \IEKF{} framework. Joint kinematic constraints, which are difficult to enforce exactly in standard formulations, are seamlessly incorporated as noise-free pseudo-measurements without sacrificing the convergence and consistency guarantees of invariant filtering. Experiments on a UR5e robot and a human leg show that the proposed \IIEKF{} consistently outperforms all \EKF{}, \IterEKF{}, and absolute-pose \IIEKF{} baselines, with faster convergence, markedly lower run-to-run variability, and better performance across both tasks and all state components. These results establish the relative $L$-extended pose as a principled foundation for IMU-based pose estimation in articulated systems. 
Future work will extend the framework to full-body human pose estimation for applications such as exoskeleton control and programming-by-demonstration. 
\changed{We will also evaluate it experimentally on spherical joints, extend it to cylindrical, prismatic, and planar joints. Finally, we will investigate strategies to compensate for soft-tissue artifacts and maintain performance under less controlled conditions.}

\appendices
\section{The Lie group of extended poses}\label{appendix}

The matrices involved in dynamics \eqref{eq:ext_pose_dyn} are given by
\begin{align}
    \boGamma &= \left[
    \begin{array}{c|c c}
        \bfI_{3} & \g{0} dt & \g{0} \frac{dt^2}{2}\\
        \hline
        \bfzero_{2\times 3} & \multicolumn{2}{c}{\bfI_{2}}
    \end{array}
    \right],\\
    \boPhi(\extpose{}{}{}) &= \left[
    \begin{array}{c|c c}
        \rot{}{} & \vel{}{}{} & \pos{}{}{} +  \vel{}{}{} dt\\
        \hline
        \bfzero_{2\times 3} & \multicolumn{2}{c}{\bfI_{2}}
    \end{array}
    \right],\\
    \boUpsilon(\inp{}{}) &= \left[
    \begin{array}{c|c c}
        \exp_{\SO{3}}(\gyr{}{}{} dt) & \acc{}{}{}dt & \acc{}{}{} \frac{dt^2}{2}\\
        \hline
        \bfzero_{2\times 3} & \multicolumn{2}{c}{\bfI_{2}}
    \end{array}
    \right],\\
    \G{k}{i}{} &= \begin{bmatrix}
        \calJ_{-\gyr{k}{i}{i} dt} & \bfzero_{3\times 3}\\
        \bfzero_{3\times 3} & \exp_{\SO{3}}(-\gyr{k}{i}{i} dt) dt\\
        \bfzero_{3\times 3} & \exp_{\SO{3}}(-\gyr{k}{i}{i} dt) \frac{dt^2}{2}
    \end{bmatrix}.
\end{align}
The Lie algebra of the group $\SE{3}{2}{}$, denoted by $\se{3}{2}{}$, is identified with $\reals^9$ via the following linear map:
\begin{equation}
    \calL_{\se{3}{2}{}}\left(\begin{bmatrix}
        \bophi\\
        \bonu\\
        \borho
    \end{bmatrix} \right) = \left[
        \begin{array}{c|c c}
            \Wedge{\bophi} & \bonu & \borho\\
            \hline
            \bfzero_{2\times 3} & \multicolumn{2}{c}{\bfzero_{2 \times 2}}
        \end{array}
        \right] \in \se{3}{2}{},
\end{equation}
where $\bophi, \bonu, \borho \in \reals^3$ and $\Wedge{\bophi}$ is the skew-symmetric matrix associated with the vector cross product in $\reals^3$.
Given $\bochi \in \SE{3}{2}{}$, the adjoint of $\bochi$ is defined as
\begin{equation}
    \Ad{\bochi} = \begin{bmatrix}
        \rot{}{} & \bfzero_{3\times 3} & \bfzero_{3\times 3}\\
        \Wedge{\vel{}{}{}}\rot{}{} & \rot{}{} & \bfzero_{3\times 3}\\
        \Wedge{\pos{}{}{}}\rot{}{} & \bfzero_{3\times 3} & \rot{}{}
    \end{bmatrix},
\end{equation}
such that $\bochi \exp(\boxi) \bochi^{-1} = \exp(\Ad{\bochi} \boxi)$ for all $\boxi \in \reals^9$.


\section{Proof of Theorem~\ref{thm:dyn}}\label{app:proof_th1}
\begin{proof}
    To simplify notation, set $\exp(\cdot)\coloneq \exp_{\SE{3}{2}{}}(\cdot)$.
    Substituting
    \eqref{eq:ext_pose_dyn} into the identity
    \begin{equation}
        \extpose{k}{p(j)s(j)}{p(j)} =
        \left(\extpose{k}{0p(j)}{0}\right)^{-1}
        \left(\extpose{k}{0s(j)}{0}\right),
    \end{equation}
    which follows from the frame-composition rule, yields:
    \begin{align*}
        \left(\X{k+1}{}{}\right)_j =& \exp(-\G{k}{p(j)}{}\w{k}{p(j)}) \boUpsilon(\inp{k}{p(j)})^{-1} \boPhi(\extpose{k}{0p(j)}{0})^{-1}\\
        & \quad \cdot \boPhi(\extpose{k}{0s(j)}{0})\boUpsilon(\inp{k}{s(j)})  \exp(\G{k}{s(j)}{}\w{k}{s(j)}),\\
        =& \exp(-\G{k}{p(j)}{}\w{k}{p(j)}) \boUpsilon(\inp{k}{p(j)})^{-1} \boPhi((\X{k}{}{})_j)\\
        & \quad \cdot \boUpsilon(\inp{k}{s(j)})\exp(\G{k}{s(j)}{}\w{k}{s(j)}),
    \end{align*}
    where we used the automorphism property $\boPhi(\bomu)^{-1}\boPhi(\bonu) = \boPhi(\bomu^{-1}\bonu)$.
    Defining $\Delta_j \coloneqq \boUpsilon(\inp{k}{p(j)})^{-1} \boPhi((\X{k}{}{})_j)\boUpsilon(\inp{k}{s(j)})$ and using the adjoint matrix $\Ad{\Delta_{j}^{-1}}$ detailed in Appendix~\ref{appendix} to move the left exponential to the right, it comes:
    \begin{align*}
        \left(\X{k+1}{}{}\right)_j=& \boUpsilon(\inp{k}{p(j)})^{-1} \boPhi((\X{k}{}{})_j)\boUpsilon(\inp{k}{s(j)})\\
        & \, \cdot \exp(- \Ad{\Delta_{j}^{-1}}\G{k}{p(j)}{}\w{k}{p(j)})  \exp(\G{k}{s(j)}{}\w{k}{s(j)}),\\
        =& \boUpsilon(\inp{k}{p(j)})^{-1} \boPhi((\X{k}{}{})_j)\boUpsilon(\inp{k}{s(j)})\\
        & \,\cdot \exp(\G{k}{s(j)}{}\w{k}{s(j)} - \Ad{\Delta_{j}^{-1}}\G{k}{p(j)}{}\w{k}{p(j)}),
    \end{align*}
    where the last line following from a first-order approximation of the Baker–Campbell–Hausdorff formula. 

    Assembling the $L$ blocks and building $\G{k}{}{}$ in such a way that $(\G{k}{}{}\w{k}{})_j = \G{k}{s(j)}{}\w{k}{s(j)} - \Ad{\Delta_j^{-1}} \G{k}{p(j)}{}\w{k}{p(j)}$ give:
    \begin{align*}
        \X{k+1}{}{} =& \boUpsilon^\boxdot(\inp{k}{p})^{-1} \boPhi^\boxdot(\X{k}{}{})\boUpsilon^\boxdot(\inp{k}{s}) \exp^\boxdot(\G{k}{}{}\w{k}{}),\\
        =& \boLambda(\X{k}{}{}, \inp{k}{p}, \inp{k}{s})\exp_{\SE{3}{2}{L}}(\bfG_k \bfw_k).
    \end{align*}
    The operator $(\cdot)^\boxdot$ preserves automorphisms, and $\boPhi^\boxdot$ satisfies $\boPhi^\boxdot(\bomu\bonu)=\boPhi^\boxdot(\bomu)\boPhi^\boxdot(\bonu)$ for all $\bomu,\bonu \in \SE{3}{2}{L}$, yielding
    \begin{align*}
        &\boLambda(\bomu\bonu, \inp{k}{p}, \inp{k}{s})\\
        &\quad = \boUpsilon^\boxdot(\inp{k}{p})^{-1} \boPhi^\boxdot(\bomu\bonu)\boUpsilon^\boxdot(\inp{k}{s}),\\
        &\quad = \boUpsilon^\boxdot(\inp{k}{p})^{-1} \boPhi^\boxdot(\bomu)\boPhi^\boxdot(\bonu)\boUpsilon^\boxdot(\inp{k}{s}),\\
        &\quad = \boUpsilon^\boxdot(\inp{k}{p})^{-1} \boPhi^\boxdot(\bomu)\boUpsilon^\boxdot(\inp{k}{s}) \boUpsilon^\boxdot(\inp{k}{s})^{-1}\\
        & \qquad\quad \cdot \boPhi^\boxdot(\bfI)^{-1}\boUpsilon^\boxdot(\inp{k}{p})\boUpsilon^\boxdot(\inp{k}{p})^{-1}\boPhi^\boxdot(\bonu)\boUpsilon^\boxdot(\inp{k}{s}),\\
        &\quad = \boLambda(\bomu, \inp{k}{p}, \inp{k}{s}) \boLambda(\bfI, \inp{k}{p}, \inp{k}{s})^{-1}\boLambda(\bonu, \inp{k}{p}, \inp{k}{s}),
    \end{align*}
    where we used $\boPhi^\boxdot(\bfI) = \bfI$. The function $\boLambda(\cdot, \inp{k}{p}, \inp{k}{s})$ thus satisfies \eqref{eq:group_affine}, proving that the dynamics are group-affine.
\end{proof}


\section{Proof of Theorem~\ref{thm:abs_constraints}}\label{app:proof_th2}
\begin{proof}
    Left-multiplying \eqref{eq:general_constraint} by $(\rot{k}{0p(j)})^T$, substituting~\eqref{eq:rel_extpose}, and rearranging terms yield
    \begin{align*}
        \rot{k}{p(j)s(j)}\bfb_k + \alpha_k \vel{k}{p(j)s(j)}{p(j)} + \beta_k \pos{k}{p(j)s(j)}{p(j)} = \bfa_k.
    \end{align*}
    Appending the identities $\alpha_k=\alpha_k$ and $\beta_k=\beta_k$ gives \eqref{eq:inv_constraint}.
\end{proof}

\bibliographystyle{IEEEtran}
\bibliography{root.bib}

\end{document}